\documentclass[10pt]{article} % For LaTeX2e

\usepackage[preprint]{rlj} % Should be uncommented for preprint versions

\usepackage{amssymb}            % Defines common symbols like \mathbb R
\usepackage{mathtools}          % Extends amsmath, providing common math tools
\usepackage{mathrsfs}           % Enables \mathscr, which can work in cases that \mathcal does not
\usepackage{graphicx}           % For including images
\usepackage{subcaption}         % Allows for the use of subfigures and subcaptions
\usepackage[space]{grffile}     % For spaces in image names
\usepackage{url}                % For displaying URLs
\usepackage{wrapfig}

\usepackage{amsfonts}
\usepackage{amsthm}
\usepackage[capitalize,noabbrev]{cleveref}
\usepackage{tikz}
\usetikzlibrary{positioning}
\usepackage{bm}
\usepackage{gmverse}
\usepackage{booktabs} 

\usepackage{algorithmic}
\usepackage{algorithm}

\newtheorem{lemma}{Lemma}[section]
\newtheorem{proposition}{Proposition}[section]

\newtheorem{definition}{Definition}[section]

\tikzstyle{vertex}=[draw,circle,minimum size=25pt,inner sep=0pt]
\tikzstyle{selected vertex} = [vertex, fill=red!24]
\tikzstyle{edge} = [draw,thick,->]
\tikzstyle{weight} = [font=\normalsize]

\title{Learning Fair Pareto-Optimal Policies in \\Multi-Objective Reinforcement Learning}

\setrunningtitle{Fair Pareto-Optimal MORL}

\author{Umer Siddique, Peilang Li, Yongcan Cao}

\emails{muhammadumer.siddique@my.utsa.edu, \\peilang.li@my.utsa.edu, yongcan.cao@utsa.edu}

\affiliations{
\textbf{Department of Electrical and Computer Engineering, \\University of Texas at San Antonio}\\
}

\contribution{
    We introduce a novel framework for fairness in multi-policy MORL, which enables learning a set of fair policies for varying user preferences. 
    }
    {
    Prior work on fairness in MORL has mainly focused on a single policy for predefined preference weights via some welfare functions. Our framework generalizes fairness across multiple policies, which allow end users to select any policy provided by their preference weights.  
    }

\contribution{
    We provide theoretical analysis demonstrating that for concave, piecewise-linear welfare functions, fair policies remain in the convex coverage set (CCS). Additionally, we establish that non-stationary and stochastic policies can enhance fairness over stationary and deterministic policies, respectively.
    }
    {
    Existing work has explored fairness in RL for predefined preference weights but has not theoretically analyzed how non-stationary and stochastic policies can improve fairness for varying preference weights.
    }

\contribution{
    We propose three scalable methods for learning fair policies in MORL using a single parameterized network: (i) an extension of Envelope~\citep{yang2019generalized} for learning fair policies, (ii) a non-stationary extension that incorporates state-augmented accrued rewards to adaptively improve fairness, and (iii) a novel stochastic policy learning method that further enhances fairness.
    }
    {
    Unlike prior work on MORL, which typically learns Pareto optimal policies, our methods efficiently learn a set of fair policies while maintaining scalability.
    }

 % Your keywords

\summary{
Fairness is important in multi-objective reinforcement learning (MORL), where policies must balance optimality and equity across objectives.
While \textit{single-policy} MORL methods can learn fair policies for fixed user preferences using welfare, they fail to generalize for different user preferences.
To address this limitation, we propose a novel framework for fairness in \textit{multi-policy} MORL, which learns a set of fair policies. 
Our theoretical analysis establishes that for concave and piecewise-linear welfare functions, fair policies remain in the convex coverage set (CCS). Additionally, we demonstrate that non-stationary and stochastic policies improve fairness over stationary and deterministic policies.
Building on our theoretical analysis, we introduce three scalable methods: an extension of Envelope for fair stationary policies, a non-stationary counterpart using state-augmented accrued rewards, and a novel extension for learning stochastic policies. We validate our methods through extensive experiments across three domains and show that our methods fairer solutions as compared to MORL baselines.
}

\newcommand{\St}{{\mathcal S}}
\newcommand{\Ac}{{\mathcal A}}
\newcommand{\Tm}{{\mathcal P}}
\newcommand{\Md}{{\mathcal M}}
\newcommand{\T}{{\bm P}}
\newcommand{\R}{{\bm r}}
\newcommand{\vvf}{{\bm V}} % vector value function, which is a matrix S x D
\newcommand{\vQ}{{\bm Q}} % 
\newcommand{\vG}{\bm G} % vectorial return 
\newcommand{\vJ}{\bm J} % vectorial objective function
\newcommand{\w}{{\bm \omega}} % GGI weights
\newcommand{\nO}{N} % number of objectives
\newcommand{\swf}{\phi} %social welfare function
\newcommand{\ggf}{\phi_{\text{GGF}}}

\usepackage{amsmath}%argmax
\DeclareMathOperator*{\argmax}{argmax}
\newcommand{\Expect}{\mathbb E}
\newcommand{\SD}{\mathbb S}

\begin{document}

\makeCover  % Create the cover page
\maketitle  % Make the title section

\begin{abstract}
Fairness is an important aspect of decision-making in multi-objective reinforcement learning (MORL), where policies must ensure both optimality and equity across multiple, potentially conflicting objectives. 
While \emph{single-policy} MORL methods can learn fair policies for fixed user preferences using welfare functions such as the \emph{generalized Gini welfare function} (GGF), they fail to provide the diverse set of policies necessary for dynamic or unknown user preferences. 
To address this limitation, we formalize the fair optimization problem in \textit{multi-policy} MORL, where the goal is to learn a set of Pareto-optimal policies that ensure fairness across all possible user preferences.
Our key technical contributions are threefold: (1) We show that for concave, piecewise-linear welfare functions (e.g., GGF), fair policies remain in the \emph{convex coverage set} (CCS), which is an approximated Pareto front for linear scalarization. (2) We demonstrate that non-stationary policies, augmented with accrued reward histories, and stochastic policies improve fairness by dynamically adapting to historical inequities. (3) We propose three novel algorithms, which include integrating GGF with multi-policy multi-objective Q-Learning (MOQL), state-augmented multi-policy MOQL for learning non-statoinary policies, and its novel extension for learning stochastic policies.
We evaluate our algorithms across various domains and compare our methods against the state-of-the-art MORL baselines.
The empirical results show that our methods learn a set of fair policies that accommodate different user preferences.
\end{abstract}

\section{Introduction}
\label{sec:intro}
Multi-objective reinforcement learning (MORL) is an important topic in the area of reinforcement learning  (RL) that focuses on designing control policies to optimize multiple objectives simultaneously. While traditional MORL methods focus on learning Pareto optimal solutions---ensuring no objective can be improved without sacrificing another---they often neglect fairness, which requires equitable treatment of all objectives or users in our context. For example, in healthcare, a policy may aim to maximize overall patient outcomes (optimality) while ensuring equal treatment across different demographic groups (fairness). A common approach to solving fairness in MORL is to use \textit{utilitarian} welfare functions, where user utilities are aggregated, typically via weighted sum, into a scalarized objective. Despite its simplicity, this approach struggles with fairness, as some users' utilities may be significantly reduced to achieve overall efficiency. An alternative approach is to employ an \textit{egalitarian} welfare function, which prioritizes the least advantaged user by maximizing the minimum utility. While this approach improves fairness, it often leads to inefficient solutions overall, as it optimizes only the lowest utility without ensuring fairness across all objectives.

Several works have explored fairness in the \textit{single-policy} RL setting~\citep{Weng19,SiddiqueWengZimmer20,zimmer2021learning,chen2021guide,do2022optimizing, fan2022welfare, YuSiddiqueWeng23ECAI, nashed2023fairness}, where a single fair policy is learned. For instance, \citet{SiddiqueWengZimmer20} enforces fairness using the GGF as a scalarized function and assigning appropriate weights to different objectives to ensure their equitable treatment. Extensions have been explored in multi-agent RL~ \citep{zimmer2021learning,siddique2024towards} and preferential treatment under known preference weights~\citep{YuSiddiqueWeng23ECAI}. Recently, fairness has been studied in multi-policy MORL~\citep{cimpeana2023multi,michailidis2024scalable} where~\citet{cimpeana2023multi} defined several fairness notions, while~\citep{michailidis2024scalable} proposed the Lorenz Condition Network (LCN), an extension of the Pareto Conditioned Network (PCN), which trains a policy network in a supervised manner to map states to desired returns. Despite these works, the investigation of fairness in RL still poses some limitations, including (1) learning a \textit{single} fair policy, (2) required knowledge of the welfare function (e.g., scalarized function) with preference weights a prior, and (3) training a conditioning network on specific return targets, limiting their ability to generalize to unseen preferences. 
Hence, existing methods operate under fixed preferences and cannot be generalized for all possible preferences.

To address these limitations, we propose a novel framework for addressing fairness in \textit{multi-policy} MORL, rather than the traditional \textit{single-policy} MORL that is the focus of existing work. Our methods are highly scalable as they leverage a single parameterized network to learn an undominated set of policies, specifically a convex coverage set (CCS), by sampling the entire preference space in MORL. In particular, to address fairness, we apply the welfare function (e.g., GGF) during learning for each sampled preference weight to ensure that each learned policy treats its objectives fairly. We further introduce non-stationary action selection using the state-augmented accrued rewards to enhance fairness by effectively utilizing historical information. We further demonstrate the benefits of learning stochastic policies for fairness. Motivated by hindsight experience replay~\citep{andrychowicz2017hindsight}, we incorporate resampling of random preference weights across different preference conditions to improve sample efficiency in MORL, as it is done in~\citep{yang2019generalized}.

The main contributions of this paper are as follows:
\begin{enumerate}
\item We introduce a novel framework for fairness in multi-policy MORL, enabling users to select any fair policy based on their specific preferences, thereby enhancing user satisfaction(~\Cref{sec:fairness}).
\item We provide theoretical analysis establishing that for concave, piecewise-linear welfare functions (e.g., GGF), fair policies remain in CCS. Additionally, we demonstrate that non-stationary policies can improve fairness by adapting to historical disparities and that stochastic policies further improve fairness over deterministic policies(~\Cref{sec:theo}).
\item Building on our theoretical insights, we propose three scalable methods for learning fair policies in MORL using a single parameterized network: (i) an extension to Envelope~\citep{yang2019generalized} for learning fair stationary policies, (ii) a non-stationary counterpart that incorporates state-augmented accrued rewards to improve fairness over time adaptively, and (iii) a novel extension for learning stochastic policies, which further enhances fairness(~\Cref{sec:algos}).
\item We experimentally validate our methods and demonstrate their effectiveness compared to state-of-the-art MORL and fairness methods across three different domains(~\Cref{sec:experimets}).
\end{enumerate}

\section{Related Work}
\label{sec: relatedwork}
Fairness in machine learning (ML) has become a significant research direction~\citep{dworkFairnessAwareness2012,zafarParityPreferencebasedNotions2017,sharifi-malvajerdiAverageIndividualFairness2019,singhPolicyLearningFairness2019,chierichettiFairClusteringFairlets2017,busa2017multi,AgarwalBeygelzimerDudikLangfordWallach18,NabiMalinskyShpitser19,zhang2021fairness}. Several studied have addressed fairness in model predictions~\citep{SpeicherHeidariGrgicHlacaGummadiSinglaWellerZafar18}, recommender systems~\citep{leonhardt2018user}, classification~\citep{dworkFairnessAwareness2012, zafarParityPreferencebasedNotions2017, AgarwalBeygelzimerDudikLangfordWallach18, kim2019multiaccuracy}, and ranking~\citep{singhPolicyLearningFairness2019}.
While much of the literature focuses on the principle of ``equal treatment of equals'', other aspects, such as proportionality~\citep{propfairness_AAMAS20} or envy-freeness~\citep{chevaleyreIssuesMultiagentResource2006} and its multiple variants (e.g.,~\citep{p292,p231}), have been considered in ML. 
In contrast, our work is grounded in distributive justice~\citep{Rawls71,
bramsFairDivisionCakeCutting1996, Moulin04}, with a focus on optimizing a welfare function for fairness considerations. This principled approach has also been recently advocated in several papers~\citep{HeidariFerrariGummadiKrause18,SpeicherHeidariGrgicHlacaGummadiSinglaWellerZafar18,malfareNeurips2021}.

Recently, fairness in RL has gained significant attention with the work by~\citep{jabbari2017fairness}, which ensures fairness in state visitation using scalar rewards. The work of \citep{Jiang2019} proposed FEN, a hierarchical decentralized approach using gossip algorithms to ensure fairness among agents. Similarly, \citet{chen2021bringing} proposed to incorporate fairness into actor-critic RL algorithms, optimizing general fairness utility functions for real-world network optimization problems. Considering the multi-objective nature of many RL problems, the study of fairness in MORL has been widely studied. In particular, \citet{SiddiqueWengZimmer20} proposed multiple adaptations to deep RL algorithms that optimize the GGF. \citet{zimmer2021learning,siddique2024fairness} extended this to the decentralized cooperative MARL. ~\citet{fan2022welfare} proposed to optimize the Nash welfare function using scalarized expected return criterion, while \citet{do2022optimizing} proposed to optimize GGF in rankings. ~\citet{YuSiddiqueWeng23ECAI,qian2025fair} proposed methods that learn a fair policy providing preferential treatment to some users while ensuring equal treatment of all others under the assumption that these preferential weights are known in advance. ~\citet{siddique2023fairness} proposed FPbRL, which learns fair preference-based policies without true rewards. Recently, fairness has been considered in multi-policy MORL with~\citet{michailidis2024scalable} propose learning Lorenz Condition networks, which ensures fairness through Lorenz domination and adds an extra parameter $
\lambda$, however, we use the welfare function to learn a set of fair optimal policies. 

Despite the significant successes achieved in the field of deep RL and MORL, existing methods heavily rely on scalarization functions to learn a \textit{single policy} with fixed preference weights. However, such single-policy methods do not work when preferences are unknown or user-specific solutions are required. To address this limitation, several works have been proposed to accommodate user-specific preferences, including but not limited to those proposed by~\citep{BarrettNarayanan08,van2013scalarized,MoffaertNowe14,yang2019generalized,alegre2023sample,reymond2022pareto}. Notably, these methods aim to learn a set of policies that approximate the Pareto frontier of optimal solutions.
For instance, ~\citep{BarrettNarayanan08} and \citep{ MoffaertNowe14} proposed methods to compute policies on the Pareto front's convex hull, while~\citep{yang2019generalized} introduced envelope Q-learning, learning policies from the convex coverage set (CCS). These approaches, however, do not address fairness, which is the focus of this paper.

\section{Preliminaries}
\label{sec:preliminaries}
\subsection{Multi-Objective Markov Decision Process}
\label{sec:mdp}
A multi-objective Markov Decision Process (MOMDP) extends the classical MDP framework to scenarios where an agent must optimize multiple objectives simultaneously. An MDP~\citep{Puterman94} is defined by the tuple, $\Md = (\St, \Ac, \Tm, r, \gamma)$, where $\St$ is the set of states, $\Ac$ is the set of actions available to the agent, $\Tm_{a,s,s'} \in [0,1]$ is the probability of transition from state $s$ to state $s'$ after taking action $a$, \textit{i.e.}, $\Tm(s'|s,a) = \Tm[S_{t+1} = s' | S_t = s, A_t = a]$, $r(s,a):s\times a\mapsto r$ is the immediate reward obtained by taking action $a$ at state $s$, and $\gamma \in [0,1)$ is the discount factor. 
An MOMDP can be represented by a tuple  $\Md = (\St, \Ac, \Tm, \R, \gamma, \Omega, f_{\Omega})$, in which the definitions of $\St, \Ac, \Tm,$ and $\gamma$ are the same as in MDP except that the reward $\R$ is now a vector, with each component corresponding to an objective that the agent seeks to optimize.
Here, the additional $\Omega$ represents the entire space of preferences, and $f_{\Omega}$ is the preference function which takes a linear form, producing a single utility $f_{\w}(\R)=\w^T \R(s,a)$, where $\w$ is a vector representing the preference weights for different objectives. 
In MOMDPs, the objectives may be conflicting, and hence it is often difficult to optimize all objectives simultaneously. 

The goal of an agent in an MOMDP is to either learn a single policy that balances multiple objectives or a set of policies that optimize different trade-offs among objectives. These approaches are referred to as \textit{single-policy} MORL and \textit{multi-policy} MORL, respectively.
A policy $\pi$ is a strategy that maps states to actions, which can be deterministic (i.e., $\forall s, \pi(s) \in \Ac$) or stochastic (i.e., $\forall s,a,\pi(a | s)$ denotes the probability of selecting $a$ in $s$).
In MOMDPs, policies are typically \textit{stationary} (Markovian), with action probabilities depending only on the current state, while a non-stationary (adaptive) policy $\pi(a|\tau, s)$ may also depend on the history $\tau$.
Standard definitions in MDPs, such as the return $G(\tau)$ and the value functions $V$ or $Q$, extend naturally to MOMDPs, albeit represented as vectors and matrices respectively. The vector return in an MOMDP is expressed as $\vG(\tau)= \sum_{t=1}^\infty \gamma^{t-1} \R_t$,
where $\tau$ is a trajectory comprising a sequence of states, actions, and rewards following the policy, and $\R_t$ is a vector reward obtained at time step $t$. The state value function of a policy $\pi$ in an MOMDP is defined as
$
\vvf^{\pi} (s) = [V^{\pi}_i (s)]= \Expect_{\tau \sim \pi}\left[\sum_{t=1}^{\infty} \gamma^{t-1} \R_t \mid S_0 = s\right],    
$
where all operations (addition, product) are applied component-wise. 

In MOMDPs, value functions do not offer a complete ordering over the policy space.
This means it is possible to encounter scenarios wherein, e.g., $V^{\pi}_i(s) > V^{\pi'}_i(s)$  for objective $i$, while $V^{\pi}_j(s) < V^{\pi'}_j (s)$ for objective $j$. 
Hence, value functions in MOMDPs induce only a partial ordering within the policy space, necessitating additional information into objective prioritization for policy ordering.

\paragraph{Envelope Multi-Objective Q-Learning.}
\label{sec:envelope}
The Envelope algorithm~\citep{yang2019generalized} learns a convex coverage set (CCS) by sampling preference weights $\w \in \Omega$ and optimizing linearly scalarized Q-values: $Q(s, a, \w) = \w^T \vQ(s, a)$, where $\vQ(s, a) \in \mathbb{R}^{\nO}$ is the vector of Q-values for $\nO$ objectives. The Bellman optimality equation for Envelope algorithm is:
$
    \vQ^*(s, a, \w) = \R(s,a) + \gamma \max_{a'} \w^T \vQ^*(s', a').
$
A single neural network parameterizes $\vQ(s, a, \w)$ by concatenating $\w$ to the state $s$, enabling efficient learning across all preferences. Despite its scalability, Envelope lacks explicit fairness guarantees, as linear scalarization may prioritize dominant objectives.

\subsection{Fairness Formulation}
\label{sec:fairness}
In MORL, fairness, rooted in distributive justice~\citep{Moulin04}, is crucial for ensuring equitable distribution of rewards.
Prior studies in fair optimization within MORL have primarily focused on learning a \textit{single-policy}, commonly referred to as an average policy~\citep{SiddiqueWengZimmer20, fan2022welfare, YuSiddiqueWeng23}. 
In this paper, we adopt a more inclusive view of fairness, including \emph{efficiency}, \emph{equity}, and \emph{impartiality} to generate fair optimal solutions for user-specific preferences. For discussion on fairness and welfare function, please refer to the Appendix.
\begin{definition}\label{def:eff}
Efficiency states that among two solutions, if one solution is (weakly or strictly) preferred by all users, then it should be preferred to the other one, e.g., $\vvf \succ \vvf^\prime \Rightarrow \swf(\vvf) > \swf(\vvf^\prime)$, where $\swf(\vvf)$ is the scalar utility function by using the $\swf$ that specifies the value of a solution.
\end{definition}
The efficiency property specifies that given all else equal, one prefers to increase a user’s utility. In the MORL setting, the efficiency property simply means Pareto dominance. More specifically, a solution is considered efficient if it is not dominated by any other solution for all objectives.

\begin{definition} \label{def:paretodominates}
For a given pair of solutions $\vvf, \vvf^\prime \in \mathbb R^\nO$, $\vvf$ \textit{weakly Pareto-dominates} $\vvf^\prime$ if $\forall i, V_i \ge V^\prime_i$, $\forall i\in\{1,\cdots, \nO \}$, where $\nO$ is the total number of objectives.
Besides, $\vvf$ \textit{Pareto-dominates} $\vvf^\prime$ if $V_i \ge V^\prime_i, \forall i$ and $\exists j, V_j > V^\prime_j$.
For brevity, we denote Pareto dominance as $\ge$ for the weak form and $>$ for the strict form.
\end{definition}

Essentially, a solution $\vvf$ (weakly) Pareto-dominates another solution $\vvf^\prime$ if the former's value $\swf(\vvf)$ (weakly) Pareto-dominates that of the latter $\swf(\vvf^\prime)$.
A solution $\vvf^*$ is said to be \textit{Pareto-optimal} if no other solution $\vvf$ Pareto-dominates it.
\textit{Pareto front} $(\mathcal{F})$ is defined as the set of Pareto-optimal solutions, which may consist of infinitely many solutions, especially when policies can be stochastic. A typical way to approximate $(\mathcal{F})$ is to compute the convex coverage set (CCS), defined below.
\begin{definition}
A solution in CCS has a maximal scalarized value in a weighted sense if there exists a weight vector $\w \in \Omega$ such that the scalarized utility $\w^T \vvf$ is weakly preferred to the scalarized utility $\w^T \vvf'$ for all other solutions $\vvf^\prime$ in the Pareto front. Formally speaking,
$\vvf \in \text{CCS} \iff \exists\ \w \in \Omega \text{ s.t. } \w^T \vvf \geq \w^T \vvf^\prime, \forall\ \vvf^\prime \in \mathcal{F}. $
\end{definition}

Next, we discuss the significance of the \textit{equity} property, a stronger property than efficiency and often associated with distributive justice, as it refers to the fair distribution of resources or opportunities.
This property ensures that a fair solution follows the \textit{Pigou-Dalton principle}~\citep{Moulin04}, which states the transferring of rewards from more advantaged users to less advantaged users. 
\begin{definition}%\label{def:pigoudalton}
A solution satisfies the \textit{Pigou-Dalton principle} if for all $\vvf$, $\vvf^\prime$ equal except for $V_i=V_i'+\delta$ and  $V_j=V_j'-\delta$ where $V_i'-V_j' > \delta > 0$, $\swf(\vvf) > \swf(\vvf^\prime)$.
\end{definition}
Finally, the \textit{impartiality} property, which is rooted in the principle of ``equal treatment of equals'' states that individuals sharing similar characteristics should be treated similarly.  

\begin{definition}
In a system, individuals with similar characteristics should be treated similarly, i.e., the solution should be independent of the order of its arguments $\swf(\vvf) = \swf(\vvf_{\sigma})$, where $\sigma$ is a permutation and $\vvf_{\sigma}$ is the vector obtained from vector $\vvf$ permuted by $\sigma$. 
\end{definition}

To ensure fairness that satisfies the above three properties, we use a well-known generalized Gini welfare function (GGF)~\citep{Weymark81}, which can be defined as:
\begin{align} \label{eq:ggi}
    \ggf(\bm u) = \sum_{i \in \nO} \omega_i  u_i^\uparrow,
\end{align}
$\bm u \in \mathbb R^\nO$ represents the utility vector of a size $\nO$ for $\nO$ objectives, $\w \in \mathbb R^\nO$ is a fixed weight vector with positive components that strictly decrease (i.e., $\omega_1 > \ldots > \omega_\nO$) with $\sum_i w_i = 1$, and $\bm u^\uparrow$ denotes the vector by sorting the components of $\bm u$ in an increasing order (i.e., $ u^\uparrow_1 \le \ldots \le  u^\uparrow_\nO$).
GGF satisfies the aforementioned three fairness properties. As the weights are positive, it is monotonic with respect to Pareto dominance, thus satisfying the efficiency property. Since the utility vector is reordered, it is also symmetric and therefore satisfies the impartiality property. Furthermore, the positive and decreasing weights ensure that GGF is Schur-concave, i.e., monotonic with respect to Pigou-Dalton transfers, therefore satisfies the impartiality property. 

GGF has been studied and used in MORL extensively~\citep{SiddiqueWengZimmer20,mandal2022sociallufair,YuSiddiqueWeng23,qian2025fair}, however, prior works have focused exclusively on the single-policy setting. To our knowledge, we are the first to apply GGF in a multi-policy MORL context. In multi-policy MORL, the standard approach is to identify all Pareto non-dominated solutions~\citep{MukaiKuroeIima12,van2014multi}; however, this is impractical for large-scale problems, as the Pareto front grows exponentially. A more scalable alternative is to approximate the CCS, which forms the convex envelope of optimal trade-offs

\section{Fairness in MORL} 
\label{sec:theo}
Since we are in a multi-policy MORL setting, where an agent learns a set of Pareto optimal policies, fairness becomes more important as different stakeholders may have different preferences, and during inference, any solution can be used from the Pareto non-dominated solutions given the stakeholder preferences. We formalize this sophisticated multi-policy fair optimization problem as: %as a family of convex optimization problems:
\begin{align} \label{eq:fair-morl}
    \forall \w \in \Omega, \quad \max_{\pi \in \Pi} \,\, \ggf(\vJ(\pi)), 
\end{align}
where $\Omega$ is the set of valid preference weights sorted in descending order, $\vJ(\pi) = \Expect_{\pi} [\sum_{t=0}^\infty \gamma^t \R_t]$ is the expected discounted return, and $\ggf(\bm{J}) = \sum_{i=1}^\nO w_i J_{(i)}$ with $J_{(1)} \leq \cdots \leq J_{(n)}$. The concavity of GGF makes problem~\eqref{eq:fair-morl} as convex optimization problem, enabling efficient solutions within the CCS. Below, we establish three foundational results, which show that it is always feasible to obtain optimal solutions in the CCS corresponding to GGF fair optimization. Next, we demonstrate that a non-stationary policy based on accrued rewards is beneficial in yielding improved fairness when compared with its stationary counterpart. Here, a policy yields improved fairness or is fairer if a higher welfare score, defined in~\eqref{eq:ggi}, is achieved. Lastly, we show that a stochastic policy may yield fairer solutions than a deterministic one.

\paragraph{\textbf{Sufficiency of Optimal Solutions in the CCS.}} 
The first question relates to the learning of fair policies in a multi-policy MORL setting is which subset of policies may be optimal among the set of all (possibly non-stationary) policies. Indeed, for linear scalarization function, CCS contains the set of Pareto front solutions. Below, we formally state it:
\begin{lemma} \label{le: op in CCS}
For any MOMDP with linear preferences over objectives,  the CCS contains an optimal policy for any linear combination of the objectives.
\end{lemma}

While GGF introduces non-linear fairness objectives, its piecewise linearity and concavity allow it to be expressed as a maximum over linear functions, which ensures that optimal solutions lie within the CCS. The following proposition establishes the sufficiency of the CCS in representing optimal policies for $\ggf$ preference weights.

\begin{proposition} \label{prop:CCS}
For any $s \in \St$ in an MOMDP and a piecewise-linear concave welfare function $\ggf$ (e.g., GGF) that can be represented as, $\ggf(\vvf^\pi(s)) = \min_{\sigma \in \SD_{\nO}} \bigl\{ \w^\top_{\sigma} \vvf^\pi(s) \bigr\}$,  there exists a policy $\pi^* \in \text{CCS}$ such that
$
    \ggf(\vvf^{\pi^*}(s)) \geq \ggf(\vvf^\pi(s)), \quad \forall \pi \in \Pi.
$
\end{proposition}

\paragraph{Example 4.1}
\begin{figure}[t]
    \centering
    \includegraphics[width=0.75\linewidth]{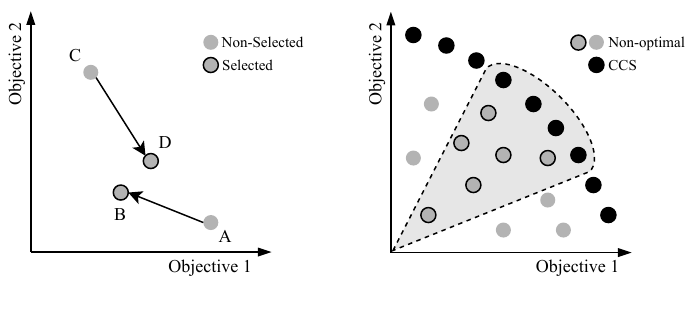}
    \caption{Examples of 2-objective MOMDP where GGF leads to fairer outcomes.}
    \label{fig:ex1}
\end{figure}
\textit{To illustrate how the GGF function ensures fairness in MORL, consider a two-objective MOMDP with objective values \(\vvf_1 = (3,1)\) and \(\vvf_2 = (2,3)\) and weights \((1,2)\). For \(\vvf_1\), two weighted combinations are possible: \textbf{A)} \((3,1) \cdot (2,1) = (6,1)\) with scalar sum \(6+1=7\), \textbf{B)} \((3,1) \cdot (1,2) = (3,2)\) with scalar sum \(3+2=5\). Since the GGF is defined as $\ggf(\vvf^\pi(s)) = \min_{\sigma \in \SD_{\nO}} \bigl\{ \w^\top_{\sigma} \vvf^\pi(s) \bigr\}$, it selects the lower scalar value, preferring point B over A (see left figure of ~\Cref{fig:ex1}). Similarly, for \(\vvf_2\): \textbf{C)} \((2,3) \cdot (1,2) = (2,6)\) with scalar sum \(2+6=8\), \textbf{D)} \((2,3) \cdot (2,1) = (4,3)\) with scalar sum \(4+3=7\).
Here, point D is preferred over C. This mechanism directs the solutions toward the fairer region (grey dotted area in the right figure of ~\Cref{fig:ex1}), demonstrating that maximizing the GGF leads to fair Pareto-optimal solutions.
}

\paragraph{\textbf{Fairness of Non-Stationary Policies.}} In fair MORL, learning non-stationary policies can be beneficial, as they use historical information to make more informed decisions and adapt over time.

\begin{proposition}\label{pro:nonstationary}
Let the reward $\R$ be nonnegative, and $\Pi_S$ and $\Pi_{NS}$ be the sets of stationary and non-stationary policies, respectively. For any $s \in \St$ in an MOMDP and a given $\ggf$, there exists a non-stationary policy $\pi_{NS} \in \Pi_{NS}$ that achieves a higher welfare score than any stationary policy $\pi_S \in \Pi_S$, i.e.,
$
    \exists \space \pi_{\text{NS}} \in \Pi_{\text{NS}} : \ggf(\vvf^{\pi_{\text{NS}}}(s)) \geq \max_{\pi_{\text{S}} \in \Pi_{\text{S}}} \ggf(\vvf^{\pi_{\text{S}}}(s)).
$
\end{proposition}

\paragraph{Example 4.2} \label{ex: non-stationary}
\textit{To illustrate the value of learning a non-stationary policy, consider a 2-objective MOMDP, shown in Fig.~\ref{fig:ex2}. At timestep $t > 0$, the agent has accrued a vector reward $\R_{\text{acc}} = (10, 0)$ for two objectives. The preference weights, encapsulated within the welfare function $\swf$, denote decreasing weights, such as $(0.8, 0.2)$. With two potential actions, each leading to a final state, action $a_1$ yields a reward of $(0, 10)$, while action $a_2$ yields $(5, 5)$. Since $s_t$ is the absorbing state, we can set the discount factor $\gamma=1$. Under the given welfare function $\swf$ defined in~\ref{eq:ggi}, executing $a_1$ yields a welfare score of $2$, whereas executing $a_2$ yields a score of $5$ if only future rewards are}
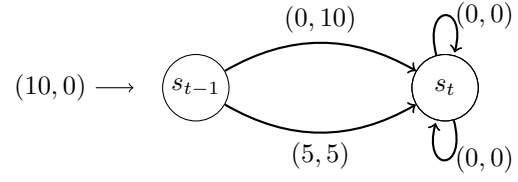
\begin{wrapfigure}{tr}{0.45\textwidth}
   \centering
   \begin{tikzpicture}[scale=1, auto,swap]
% \footnotesize
   % Define placeholder nodes
   \node (x2) {}; % Removed 'vertex' style
   \node[vertex] (x3) at (4,0) {};
   % Accrued past reward
   \node[left=0.5cm of x2] (accrued) {$(10,0)$};
   \draw[->] (accrued) -- (x2);
   % First we draw the vertices
   \foreach \pos/\name/\label in {{(0.7,0)/x2/s_{t-1}}, {(4,0)/x3/s_t}}
       \node[vertex] (\name) at \pos {$\label$};
   % Connect vertices with edges and draw weights
   \foreach \source/\dest/\weight in {x2/x3/{(0,10)}}
       \path[edge] (\source) edge [bend left] node[weight,above] {$\weight$} (\dest);
   \foreach \source/\dest/\weight in {x2/x3/{(5,5)}}
       \path[edge] (\source) edge [bend right] node[weight,below] {$\weight$} (\dest);
   % Additional edges for s_t
   \foreach \source/\dest/\weight in {x3/x3/{(0,0)}}
       \path[edge] (\source) edge [loop above] node[weight,right] {$\weight$} (\dest);
   \foreach \source/\dest/\weight in {x3/x3/{(0,0)}}
       \path[edge] (\source) edge [loop below] node[weight,right] {$\weight$} (\dest);
\end{tikzpicture}

\caption{Example of MOMDP where actions lead to different rewards.}
\label{fig:ex2}
\end{wrapfigure}
\textit{considered.  However, considering historical data, i.e., $\R_{\text{acc}}$, $a_1$ yields a higher accrued episodic return of $(10, 10)$ and a welfare score of $10$. Similarly, $a_2$ yields $(15, 5)$ and $7$ episodic return and welfare scores, respectively. 
Note that action $a_1$ is a fairer choice in this case since it balances the two objectives, unlike action $a_2$, which fails to achieve a more equitable outcome. Hence, employing historical data, namely, accrued rewards in this case, is critical to enable fair policy learning. }

\paragraph{\textbf{Optimality of Stochastic Policies for Fairness}}
Unlike single-objective RL, in MORL, a deterministic policy may not be optimal. A fairer solution can often be achieved through randomization.

\begin{proposition} \label{prop:stochastic}
Let $\Pi_{\text{ST}}$ be the set of stochastic policies and $\Pi_{\text{D}}$ be the set of deterministic policies.
For an MOMDP $\Md$ and a concave welfare function such as $\ggf$, there exists a stochastic policy $\pi_{\text{ST}} \in \Pi_{\text{ST}}$ such that
$
\ggf(\vvf^{\pi_{\text{ST}}}) \geq \max_{\pi_{\text{D}} \in \Pi_{\text{D}}} \ggf(\vvf^{\pi_{\text{D}}}).
$
\end{proposition}

\begin{figure}[h]
    \centering
    \includegraphics[width=0.75\linewidth]{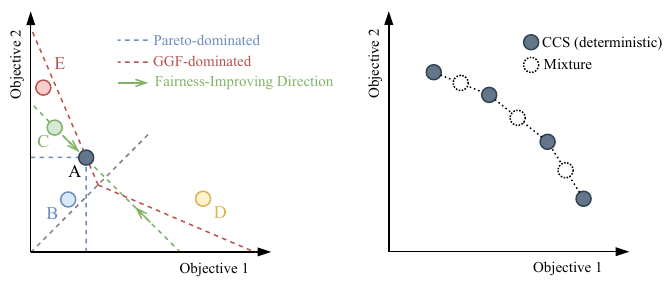}
    \caption{Left Figure: Point A Pareto-dominates B and is preferred to C by the Pigou–Dalton transfer(fairer solution). Depending on GGF weights, D and E may be dominated or non-dominated by A (w.r.t. GGF); for weights~$(0.3, 0.7)$, A is preferred to E but not D. Right: Black points denote deterministic policies in the CCS; mixing these yields stochastic policies (dotted points), which can achieve fairer solutions unattainable by any single deterministic policy.}
    \label{fig:ex3}
\end{figure}

Proofs of the above lemma and propositions are provided in~\Cref{app:proofs}. The left figure of \Cref{fig:ex3} illustrates GGF in a two-objective task. The optimality of stochastic policies implies that restricting the search to deterministic policies is insufficient, and stochastic policies expand the solution space and can better capture trade-offs, thus improving overall fairness, as shown in~\Cref{fig:ex3}.

\section{Proposed Algorithms} 
\label{sec:algos}
In this section, we introduce three novel algorithms that incorporate fairness into MORL based on the technical analysis in the previous section. 
These algorithms optimize the GGF function~\eqref{eq:ggi} to ensure fairness across $\nO$ fixed users with varying preferences.
These methods are scalable and sample-efficient as they utilize a single parameterized network to estimate Q-values for all objectives while maintaining a diverse set of Pareto-optimal policies.
Specifically, we introduce Fair Multi-Objective Deep Q-Learning (F-MDQ), its non-stationary extension (FN-MDQ), and a novel extension incorporating stochastic policies (FNS-MDQ). This progression from stationary to non-stationary to stochastic and non-stationary policies demonstrates our systematic approach to enhancing fairness in MORL algorithms, with each method building upon the previous one.

\paragraph{F-MDQ.} F-MDQ builds on the Envelope algorithm~\citep{yang2019generalized} by replacing the linear scalarization function with the GGF welfare function $\swf$. This ensures fairness while learning policies across all preferences $\w \in \Omega$. The Bellman optimality equation for F-MDQ is given by:
\begin{align*}
    \vQ^*(s, a, \w) = \Expect [\R(s,a) + \gamma \vQ^*(s', \sup_{a' \in \Ac} 
     \ggf(\R(s,a) + Q^*(s',a',\w), \w) \mid s,a ],
\end{align*}
where $\vQ^\pi(s, a, \w)$ represents the expected return vector for policy $\pi$, conditioned on preference $\w$. As the MO Q-function is parameterized, it can be learned by minimizing the loss function
$
    \mathcal{L} = \Expect_{(s,a,\R,s',\omega) \sim \mathcal{D}} \left[ \| \bm y - \bm Q(s,a,\w) \|_2^2 \right],
$
where the expectation is taken over experiences sampled from the replay buffer $\mathcal{D}$. Given that the loss function includes an expectation over $\w$, the preference weights are sampled randomly and are decoupled from the transitions, allowing increased sample efficiency through a resampling scheme similar to Hindsight Experience Replay (HER)~\citep{andrychowicz2017hindsight}. The target $\bm y$ is F-MDQ is computed as 
\begin{align*}
    \bm y = \R(s,a) + \gamma \vQ'(s', \sup_{a' \in \Ac} \ggf(\R(s,a) + \gamma Q(s',a',\w)), \w),
\end{align*}
where $Q'$ represents the target multi-objective Q-function, and the supremum is applied over the GGF welfare function $\ggf$ instead of a linear weighted sum. This ensures that actions are selected based on higher welfare scores rather than simply maximizing Q-values.

\paragraph{FN-MDQ.}
FN-MDQ extends F-MDQ by incorporating accrued rewards into the state to learn non-stationary policies, as discussed in Proposition~\ref{pro:nonstationary}. It augments the state with accrued rewards, allowing the agent to balance reward distribution across users (as demonstrated in Example 2). The augmented state is defined as $\mathfrak{s}_t = (s_t, \R_{\text{acc}})$, where $\R_{\text{acc}} =  \sum_{i=1}^{t-1}\gamma^{i-1} \R_i$ is the discounted reward received in the current trajectory. The regression target for FN-MDQ is given by
\begin{align*}
\R(s_t, a_t) + \gamma \vQ'(\mathfrak{s}_{t+1}, \sup_{a' \in \Ac} \ggf(Q(\mathfrak{s}_{t+1}, a', \w)), \w).
\end{align*}
Here, the immediate reward $\R(s_t, a_t)$ is excluded from the optimal action since this is already included in the augmented state as part of the discounted total reward. This extension enables the agent to identify and prioritize users who have received insufficient rewards within a trajectory.

\paragraph{FNS-MDQ.}
Given that stochastic policies can outperform deterministic ones (as established in~\Cref{prop:stochastic}), the performance of FN-MDQ can be enhanced by incorporating stochastic policies. We now explain how stochastic policies can be integrated into the FN-MDQ algorithm.

Under the stochastic policies, the target Q-value is adjusted to account for the expected Q-values, which reformulates the update as
\begin{align*}
\R(s_t, a_t) + \gamma \vQ'(\mathfrak{s}_{t+1}, \sum_{a' \in \Ac} \ggf(\pi(a' \mid \mathfrak{s}_{t+1}) Q(\mathfrak{s}_{t+1}, a', \w)), \w),
\end{align*}
where $\pi(a' \mid \mathfrak{s}_{t+1})$ is the probability of taking action $a'$ given the augmented state $\mathfrak{s}_{t+1}$. This reformulation considers the distribution of possible actions rather than selecting a single best deterministic action, aligning with our theoretical insights.

Unlike F-MDQ and FN-MDQ, which rely on deterministic action selection, FNS-MDQ samples actions from a probability distribution over Q-values. This stochastic action selection improves fairness by enabling more balanced policy exploration and reducing biases that arise from always selecting the highest Q-value action. Note that, during the training phase, all algorithms employ an $\epsilon$-greedy policy during training, however, FNS-MDQ differs in its action-selection strategy by using the best learned stochastic policy rather than a deterministic greedy approach. This increased flexibility and randomness can lead to more equitable solutions.

\section{Experiments} \label{sec:experimets}
To evaluate the proposed methods, we conduct experiments across three domains---each characterized by varying levels of complexity in terms of the number of objectives. These domains, ranging from low to high in terms of the number of objectives, include species conservation, resource gathering, and multi-product web advertising. Each environment presents unique challenges where fairness plays a critical role. We first briefly describe each environment (details are available in Appendix B) and then present our experimental results.

\subsection{Environments}
Our first domain is a species conservation (SC) environment, which addresses a critical ecological challenge: balancing the populations of two highly interacting endangered species, the sea otter and the northern abalone. Both species are at risk of extinction, requiring sophisticated management strategies to ensure their survival. We adopt the model proposed by~\citep{chades2012setting}, which simulates the predation relationship between the species, where sea otters prey on abalones. This dynamic presents a unique preservation challenge, as the survival of one species could potentially drive the other to extinction if not properly managed. The state space is composed of the current population sizes of sea otters and northern abalones. The action space includes introducing sea otters, enforcing anti-poaching measures, controlling sea otter populations, implementing a combination of half-antipoaching and half-controlled sea otters, or taking no action. Each action has significant ecological implications. For instance, introducing sea otters may help balance the abalone population, but if mismanaged, could lead to abalone extinction. The reward function is defined by the population densities of both species, i.e., $\nO=2$. Fairness in this context is interpreted as achieving a balanced distribution of species densities to ensure their preservation.

Our second environment is a resource-gathering (RG) problem, which is a $5 \times 5$ grid world that contains three types of resources: gold, gems, and stones. These resources are randomly positioned on the grid and regenerate randomly upon consumption. The main challenge here is to collect these resources, where each resource has a different value: gold and gems are valued at 1, while stones have a lower value of 0.4. This creates an intentionally uneven resource distribution, with two stones, one gold, and one gem. In this environment, the state is defined by the agent's current location on the grid and the cumulative count of each resource collected during its trajectory. The agent can take four actions: up, down, left, and right. The reward is a vector representing the resources collected for each type. In this environment, fairness is defined as the equitable collection of resources, despite their differing values. Note that this problem is particularly important for validating whether the proposed methods can achieve fairer solutions while still reaching Pareto optimal solutions.

Our third domain is a multi-product web advertising (MWP) problem that involves an online store offering $\nO=7$ distinct products. Here, the agent decides which advertisement to display: a product-specific advertisement for one of the products $i \in [0, ..., \nO-1]$, or a general advertisement that is not tailored to any specific product. In this environment, the state space includes the number of products available in the store, as well as the number of visits, purchases, and exits. The action space is $\nO+1$, where actions $0$ through $\nO-1$ correspond to displaying advertisements for specific products, and action $\nO$ involves showing a general advertisement. This additional action adds complexity, requiring the agent to decide the optimal moment to transition between states. The reward function is designed so that the agent receives a reward of 1 in the $i^{th}$ dimension of the reward vector if a product of the type $i$ is sold after displaying its advertisement. In this environment, fairness is defined as balancing the frequency of advertisements shown for each product, ensuring no single product is overly prioritized. The challenge lies in increasing overall rewards while maintaining a fair distribution of advertisement exposure across all products.

\begin{figure*}[t]
    \centering
    \begin{subfigure}[t]{0.49\linewidth}
        \centering
	     \includegraphics[width=\linewidth]{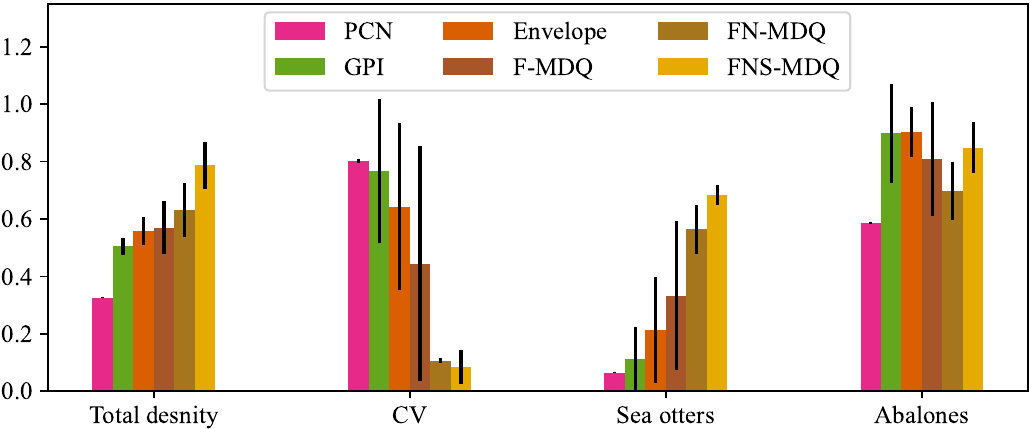}
	      \caption{Total density, CV, min density, and max density.}
	      \label{fig:spec_cv}
	\end{subfigure}
	\begin{subfigure}[t]{0.49\linewidth}
	    \centering
         \includegraphics[width=\linewidth]{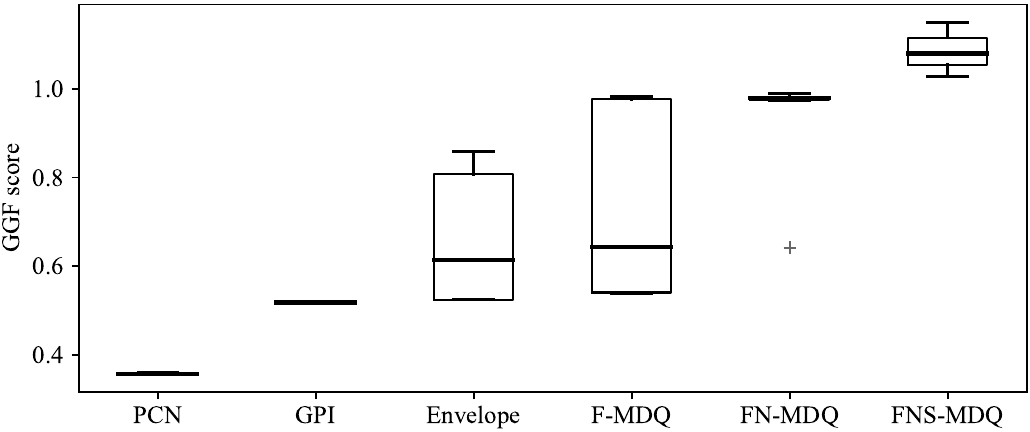}
        \caption{GGF scores.}
        \label{fig:spec_hv}
    \end{subfigure}
    \caption{Performances of multi-policy MORL baselines and our methods in species conservation.}
    \label{fig:spec}
\end{figure*}

\begin{figure*}[t]
    \centering
    \begin{subfigure}[t]{0.245\linewidth}
        \centering
	     \includegraphics[width=\linewidth]{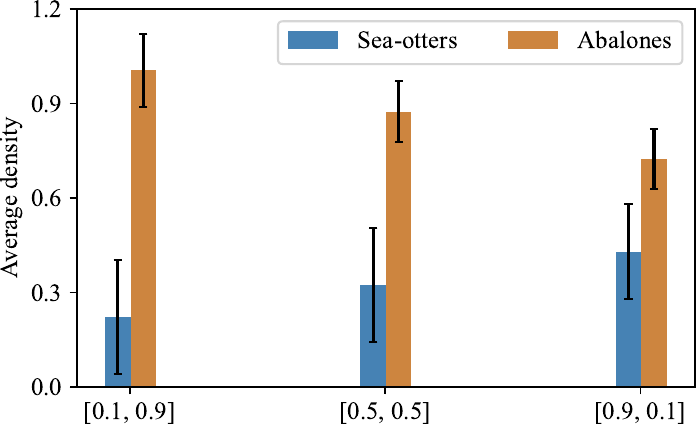}
	      \caption{Envelope.}
	      \label{fig:sc_bar_en}
	\end{subfigure}
	\begin{subfigure}[t]{0.245\linewidth}
	    \centering
         \includegraphics[width=\linewidth]{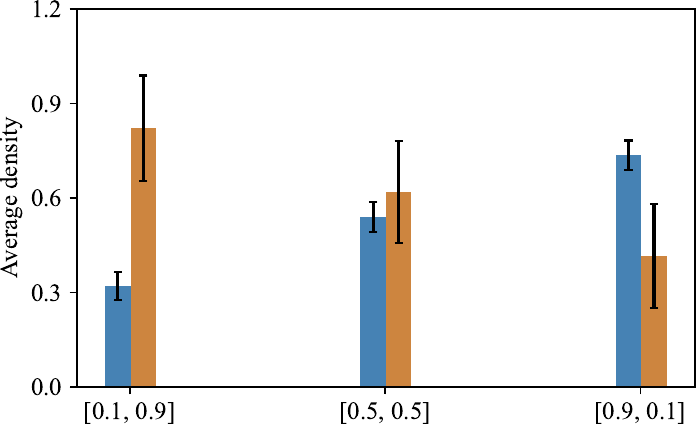}
        \caption{F-MDQ.}
        \label{fig:sc_bar_f}
    \end{subfigure}
    \begin{subfigure}[t]{0.245\linewidth}
	    \centering
         \includegraphics[width=\linewidth]{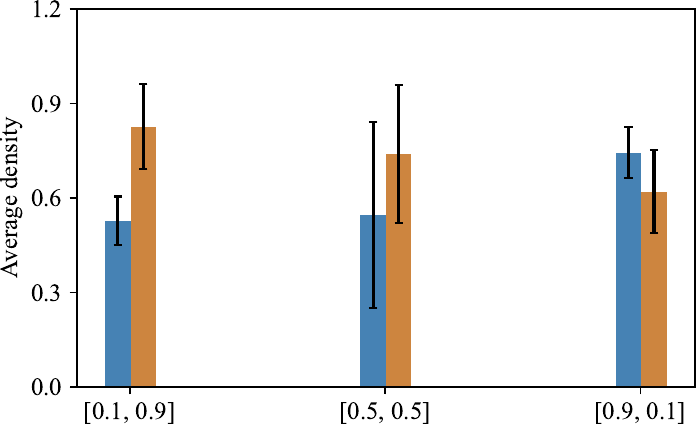}
        \caption{FN-MDQ.}
        \label{fig:sc_bar_fn}
    \end{subfigure}
    \begin{subfigure}[t]{0.245\linewidth}
	    \centering
         \includegraphics[width=\linewidth]{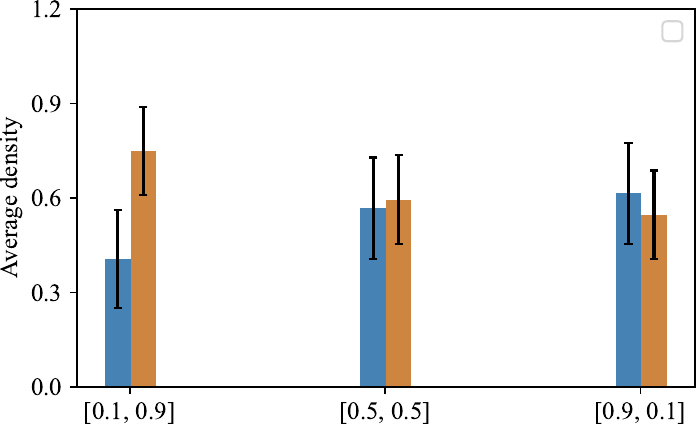}
        \caption{FNS-MDQ.}
        \label{fig:sc_bar_fns}
    \end{subfigure}
    \caption{Individual densities of Envelope, and our proposed methods during testing with unseen preferences in species conservation.}
    \label{fig:sc_ind_bar}
\end{figure*}

\subsection{Baselines}
We compare our proposed methods against several multi-policy MORL baselines. Generalized Policy Improvement Linear Support (GPI-LS) \citep{alegre2023sample} employs GPI~\citep{barreto2017successor} to combine policies within its learned CCS and prioritize the weight vectors on which agents should train at each moment. The Envelope~\citep{yang2019generalized} uses a single neural network conditioned on a weight vector to approximate the CCS. PCN~\citep{reymond2022pareto} utilizes a neural network conditioned on a desired return per objective and is trained via supervised learning to predict actions that yield the desired return.
Hyperparameters for each method were optimized, and experiments were run for five different seeds, with average results reported. Further details on experimental configurations and hyperparameters are provided in Appendix C.

\subsection{Results}
In this section, we present the experimental results across the three environments presented above. The primary objective of these experiments is to assess the effectiveness of our proposed methods by addressing the following key research questions: \textbf{(A)} How effective are our methods in learning fairer solutions compared to multi-policy MORL baselines? \textbf{(B)} Can our methods generate fair solutions across different preference settings during inference? 
\textbf{(C)} To what extent can our proposed algorithms achieve comparable performance in terms of hypervolume and cardinality relative to multi-policy MORL approaches? 
\textbf{(D)} What is the impact of our approach on the diversity and quality of non-dominated solutions that satisfy fairness criteria? \textbf{(E)} Does the incorporation of stochastic policies in MO Q-learning based algorithms contribute to improved fairness or overall performance?

\paragraph{\textbf{Question (A)}} To evaluate how effective our methods are in learning fair solutions, we conducted experiments in the SC, RG, and MWP domains, as shown in~\Cref{fig:spec_cv,fig:rc_cv,fig:mv_cv}. We compare our proposed methods (F-MDQ, FN-MDQ, and FMS-MDQ) with multi-policy MORL baselines such as PCN, GPI, and Envelope during the training phase. 
We choose these baselines as they are the current state-of-the-art MORL baselines. The Key evaluation metrics used include total rewards, Coefficient of Variation (CV) indicating the variations in different objectives’ utilities, and the minimum and maximum objective utilities. Moreover, GGF welfare scores were computed to quantify fairness. As we are in a multi-policy MORL, an agent learns a set of Pareto optimal policies during learning. To show the results, we computed these metrics over the last 50 trajectories for all the Pareto optimal policies and reported their normalized scores. Note that, during the last 50 trajectories, all the agents are converged so it ensures a fair comparison for multi-policy MORL methods. 

As shown in~\Cref{fig:spec_cv}, PCN performs the worst. GPI outperforms PCN, likely due to its TD3-based~\citep{fujimototd32018} architecture and efficient prioritization scheme in learning the Pareto front $\mathcal{F}$. The Envelope algorithm performs better than PCN and GPI as it achieves higher total density and, interestingly, lower CV. However, our proposed algorithms outperform all other methods by achieving the lowest CV and highest welfare scores~\Cref{fig:spec_hv}, with FN-MDQ outperforming F-MDQ, underscoring the value of non-stationary policies. Furthermore, FNS-MDQ outperforms both F-MDQ and FN-MDQ as it maximizes the minimum objective utility and demonstrates better fairness through optimizing the welfare function $\ggf$. Similar results are observed in RG~\Cref{fig:rc_cv}, where PCN performs the worst as it collects the least resources, likely due to its limitations in deterministic environments~\citep{reymond2022pareto}. Although GPI performs better than PCN, both exhibit low CV alongside poor overall performance and GGF welfare utility~\Cref{fig:rc_hv}. The Envelope algorithm achieves better performance in terms of rewards but suffers from the highest CV and lower GGF utility scores. In contrast, our proposed methods attain a lower CV compared to all baselines, and they achieve the highest GGF scores, highlighting their effectiveness in identifying fair policies through welfare function optimization. Interestingly, FNS-MDQ exhibits a higher CV due to its higher maximum objective and the total resources collected. Nevertheless, it also achieves the highest welfare scores. Consistent with our previous results, our proposed methods in MVP environment~\Cref{fig:mv_cv} achieve the highest welfare scores, indicating their capacity to ensure an equitable distribution of rewards across all objectives. Moreover, they maintain the lowest CV, highlighting their robustness in learning fair policies, even in highly stochastic environments with a higher number of objectives. Once again, PCN, and GPI perform the worst, further underscoring the efficacy of our methods in this context.

\paragraph{\textbf{Question (B)}} % Can our methods generate fair solutions across different preference settings during inference? 
To check whether our methods can generate fair solutions across different preference settings, we evaluated our algorithms with unseen preferences during testing in the SC environment. As shown in~\Cref{fig:sc_ind_bar}, which presents the individual species densities (sea otters and abalones) for preference configurations $(0.1,0.9), (0.5, 0.5), (0.9, 0.1)$, the Envelope algorithm fails to produce fair solutions, suggesting its limitation in generating fair optimal policies across varying preferences.  In contrast, F-MDQ generates more balanced solutions, while FN-MDQ and FNS-MDQ achieve even fairer outcomes, further validating our earlier findings.

\begin{figure*}[t]
    \centering
    \begin{subfigure}[t]{0.49\linewidth}
        \centering
	     \includegraphics[width=\linewidth]{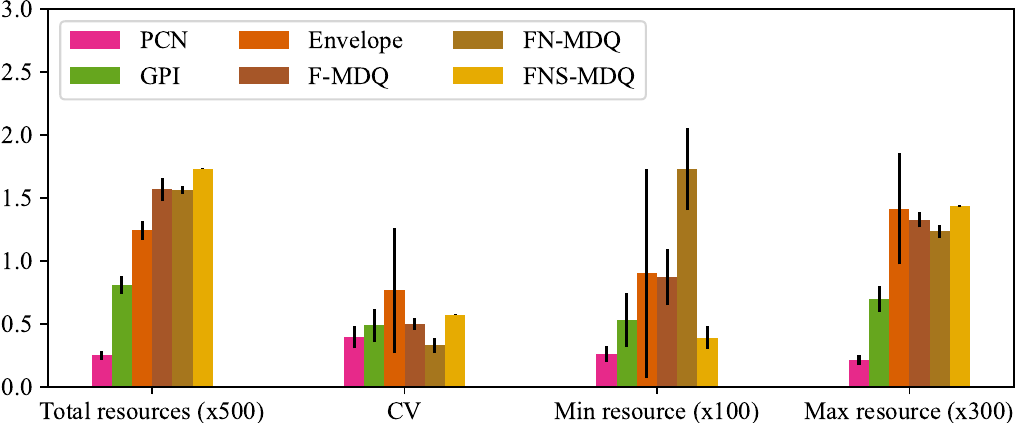}
	      \caption{Total resources, CV, min and max resources.}
	      \label{fig:rc_cv}
	\end{subfigure}
	\begin{subfigure}[t]{0.49\linewidth}
	    \centering
         \includegraphics[width=\linewidth]{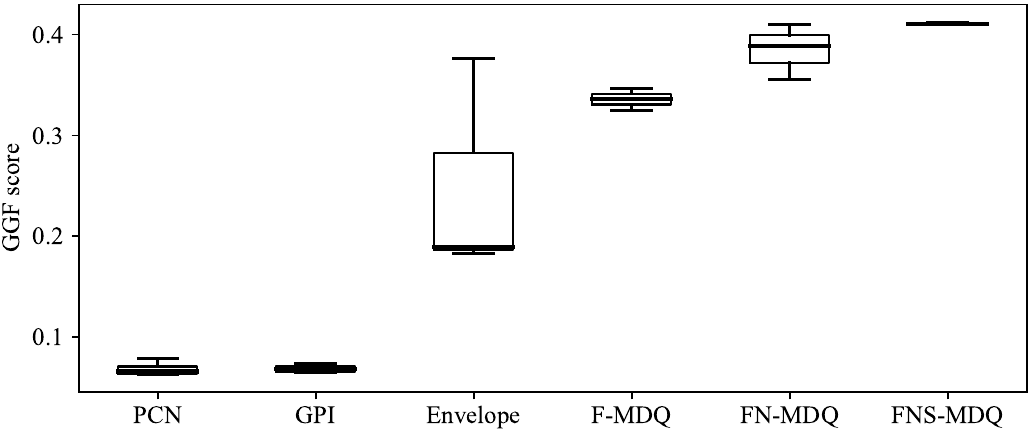}
        \caption{GGF scores.}
        \label{fig:rc_hv}
    \end{subfigure}
    \caption{Performances of multi-policy MORL baselines and our methods in resource gathering.}
    \label{fig:rc}
\end{figure*}
\begin{figure*}[t]
    \centering
    \begin{subfigure}[t]{0.49\linewidth}
        \centering
	     \includegraphics[width=\linewidth]{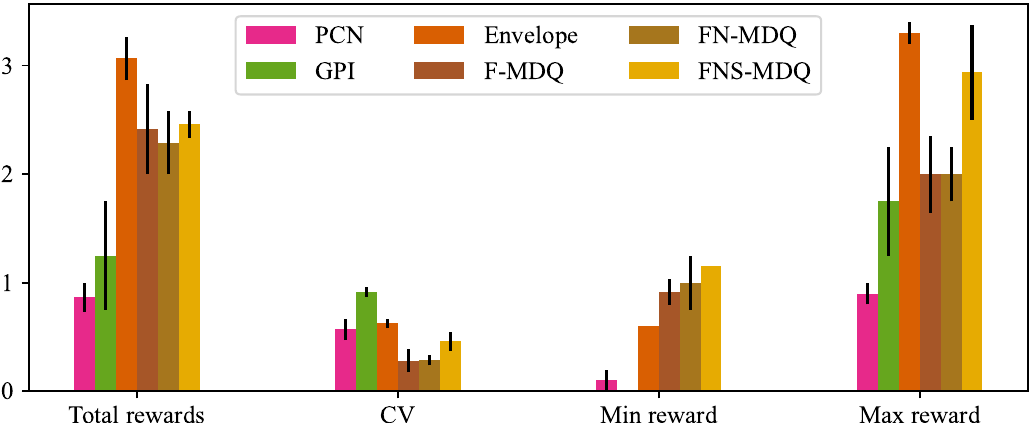}
	      \caption{Total rewards, CV, min reward, and max reward.}
	      \label{fig:mv_cv}
	\end{subfigure}
	\begin{subfigure}[t]{0.49\linewidth}
	    \centering
         \includegraphics[width=\linewidth]{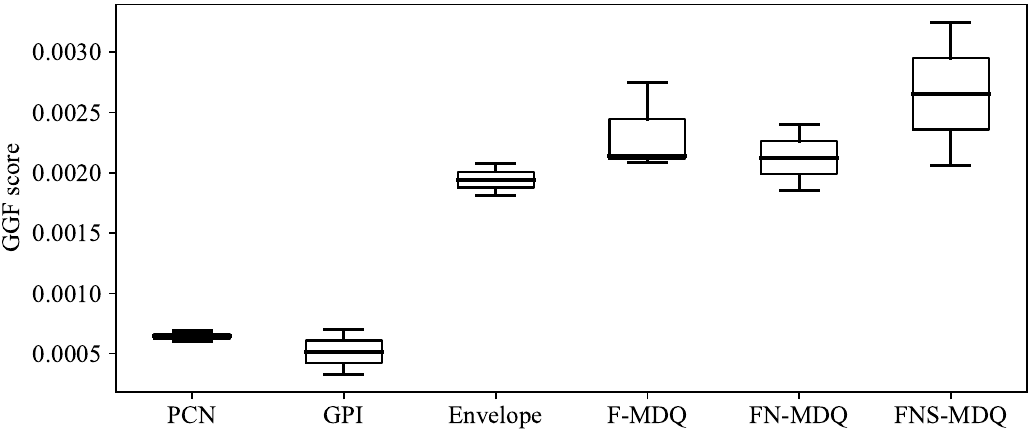}
        \caption{GGF scores.}
        \label{fig:mv_hv}
    \end{subfigure}
    \caption{Performances of multi-policy MORL baselines and our proposed methods in the MPW.}
    \label{fig:mv}
\end{figure*}

\paragraph{\textbf{Question (C)}} %To what extent can our proposed algorithms achieve comparable performance in terms of hypervolume and cardinality relative to multi-policy MORL approaches?
To answer this question, we evaluate algorithms in terms of MORL metrics, such as cardinality and hypervolume (HV). A higher cardinality indicates greater policy diversity within $\mathcal{F}$, while HV measures both the convergence rate and policy diversity~\citep{LaumannsThieleDebZitzler02}. Recall that, HV is defined as for any given $\mathcal{F}^\prime$ an approximation of $\mathcal{F}$ and a reference point (the worst-possible return), it measures the volume of the hypercube spanned by the reference point and estimated return in a trajectory. ~\Cref{tab:hv_cd_all} presents the HV and cardinality in all environments. These results show that our proposed methods perform on par with the considered baselines.

\paragraph{\textbf{Question (D)}} %What is the impact of our approach on the diversity and quality of non-dominated solutions that satisfy fairness criteria?
The results discussed in previous questions suggest that our methods can generate a range of Pareto optimal solutions across varied preference configurations, which indicates better coverage of the objective space, thus leading to improved performance across multiple objectives. For quality, our proposed algorithms consistently achieve the lowest  CV and highest GGF welfare scores across SC, RG, and MVP domains, indicating that our solutions exhibit more equitable distribution of objective utilities while maintaining Pareto optimality compared to baseline methods (PCN, GPI, and Envelope). These outcomes align with our theoretical justifications (see~\Cref{sec:theo}).

\begin{table*}[t]
\centering
\caption{Hypervolume (HV) and Cardinality (CD) of various MORL methods on SC, RC, and MWP.}  
\vspace{8pt}
\label{tab:hv_cd_all}
\small
\begin{tabular}{lcc|cc|cc}  
\toprule
Methods  & \multicolumn{2}{c|}{SC}  & \multicolumn{2}{c|}{RC}  & \multicolumn{2}{c}{MWP}  \\  
\midrule
         & HV $(10^4)^\uparrow$ & $\text{CD}^\uparrow$ & HV ($10^5$) & CD & HV ($10^9$) & CD \\  
\midrule
PCN      & $1.81 \pm 0.14$ & $19.67 \pm 2.99$ & $11.69 \pm 0.90$ & $6.0 \pm 1.27$ & $10.17 \pm 0.22$ & $43.5 \pm 1.06$ \\
GPI      & $2.82 \pm 0.03$ & $12.0 \pm 2.05$ & $7.33 \pm 0.19$ & $43.0 \pm 2.62$ & $10.44 \pm 0.86$ & $41.0 \pm 2.83$ \\
Envelope & $2.35 \pm 0.18$ & $5.6 \pm 1.04$ & $17.51 \pm 3.73$ & $19.75 \pm 6.79$ & $10.55 \pm 1.96$ & $51.5 \pm 1.06$ \\
F-MDQ    & $2.22 \pm 0.19$ & $6.6 \pm 1.31$ & $16.92 \pm 1.63$ & $31.33 \pm 7.84$ & $10.45 \pm 2.40$ & $48.0 \pm 2.12$ \\
FN-MDQ   & $2.34 \pm 0.07$ & $11.68 \pm 1.05$ & $20.38 \pm 1.49$ & $33.54 \pm 8.29$ & $10.51 \pm 2.42$ & $52.2 \pm 2.44$ \\
FNS-MDQ  & $2.91 \pm 0.20$ & $15.38 \pm 1.10$ & $24.40 \pm 2.22$ & $36.11 \pm 8.96$ & $10.62 \pm 2.45$ & $51.05 \pm 2.30$ \\
\bottomrule
\end{tabular}
\end{table*}

\paragraph{\textbf{Question (E)}} %Does the incorporation of stochastic policies in MO Q-learning algorithms contribute to improved fairness or overall performance?
Finally, to assess the impact of incorporating stochastic policies in MO Q-learning algorithms, we refer to the results in~\Cref{fig:spec_cv,fig:rc_cv,fig:mv_cv}, where stochastic policies consistently improve both efficiency and fairness. Moreover, as shown in~\Cref{tab:hv_cd_all} incorporating stochastic policies also enhances MORL metrics, including HV and cardinality, 
validating the contribution of stochasticity to both fairness and overall performance.

\section{Conclusions and Limitations}
In this paper, we presented a novel approach to addressing fairness in the context of multi-policy MORL. Our proposed methods leverage a single parameterized network to learn optimized policies across the entire space of possible preferences. Both theoretical and empirical analyses demonstrate that learning a non-stationary policy significantly improves fairness. Additionally, we highlighted the importance of stochastic policies in achieving fair outcomes. Experimental evaluations in three domains validated the effectiveness of our approach in yielding more equitable policies compared to state-of-the-art MORL and fair baselines.

Our approach also has some limitations. First, it is limited to MOMDPs with discrete action spaces. Second, it assumes that preference weights are linear to learn the CCS, which may not capture the concave regions of the Pareto front. Third, the current formulation is focused on individual fairness. Given that optimizing a welfare function is a broad framework applicable to various real-world MORL problems involving general utilities, an important direction for future research is to extend this approach to accommodate more sophisticated objective functions, particularly those related to group-level fairness, safety, and risk sensitivity.

%%%%%%%%%%%%%%%%%%%%%%%%%%%%%%%%%%%%%%%%%%%%%%%%%%%%%%%%%%%%%%%%
%% Section: Submission of papers to RLJ/RLC
%%%%%%%%%%%%%%%%%%%%%%%%%%%%%%%%%%%%%%%%%%%%%%%%%%%%%%%%%%%%%%%%

\subsubsection*{Acknowledgments}
\label{sec:ack}
This work was supported by the Office of Naval Research under Grant N000142412405 and the Army Research Office under Grants W911NF2110103 and W911NF2310363.

%%%%%%%%%%%%%%%%%%%%%%%%%%%%%%%%%%%%%%%%%%%%%%%%%%%%%%%%%%%%%%%%
%% NOTE: THIS MARKS THE END OF THE "MAIN TEXT"
%%%%%%%%%%%%%%%%%%%%%%%%%%%%%%%%%%%%%%%%%%%%%%%%%%%%%%%%%%%%%%%%

%%%%%%%%%%%%%%%%%%%%%%%%%%%%%%%%%%%%%%%%%%%%%%%%%%%%%%%%%%%%%%%%
%% Bibliography
%%%%%%%%%%%%%%%%%%%%%%%%%%%%%%%%%%%%%%%%%%%%%%%%%%%%%%%%%%%%%%%%
\bibliography{main}
\bibliographystyle{rlj}

%%%%%%%%%%%%%%%%%%%%%%%%%%%%%%%%%%%%%%%%%%%%%%%%%%%%%%%%%%%%%%%%
% AUTHOR: If your paper has no supplementary materials, you may 
%         comment out the line below, which creates the title for
%         the supplementary materials.
%%%%%%%%%%%%%%%%%%%%%%%%%%%%%%%%%%%%%%%%%%%%%%%%%%%%%%%%%%%%%%%%
\beginSupplementaryMaterials

\section{Proofs of Technical Analysis}
\label{app:proofs}
In this section, we provide formal proofs of our technical analysis in detail. For better legibility, we first recall the equations and results that we need for our proofs.
\begin{align}
    \forall \w \in \Omega, \quad \max_{\pi \in \Pi} \,\, \swf(\vvf(\pi)), 
\end{align}
where $\Omega$ is the set of valid preference weights sorted in descending order, $\vvf(\pi) = \Expect_{\pi} [\sum_{t=0}^\infty \gamma^t \R_t]$ is the expected discounted return, and $\swf(\bm{J}) = \sum_{i=1}^\nO w_i \vvf_{(i)}$ with $\vvf_{(1)} \leq \cdots \leq \vvf_{(n)}$. 

\begin{lemma} \label{le: op in CCS}
For any MOMDP with linear preferences over objectives,  the CCS contains an optimal policy for any linear combination of the objectives.
\end{lemma}

\begin{proof}
Let $\St$ be the state space, $\Ac$ be the action space, and $\R: \St \times \Ac \rightarrow \R^{\nO}$ be the vector-valued reward function, where $\nO$ is the number of objectives. Consider a linear preference vector $\w \in \Omega$, where $\Omega = \{ \w \in \R^{\nO} : \sum_{i=1}^{\nO} w_i = 1, w_i \geq 0 \}$. For any policy $\pi$, the expected return under a preference $\w$ is given by $\w \big (\Expect_{\pi}\left[\sum_{t=1}^{\infty} \gamma^{t-1} \R(s_t, a_t) \mid s_0 = s\right] \big)$. Thus, the optimal policy $\pi^*_{\w}$ for preference $\w$ satisfies
\begin{align*}
    \pi^*_{\w} = \argmax_{\pi} \w^T \vvf^{\pi}(s),~~ \forall s \in \St.
\end{align*}
By the definition of the CCS, for any $\w \in \Omega$, there exists a policy $\pi_{\text{CCS}} \in \text{CCS}$ such that
\begin{align*}
    \w^T \vvf^{\pi_{\text{CCS}}}(s) \geq \w^T \vvf^{\pi}(s), ~~ \forall \pi \in \Pi, \forall s \in \St.
\end{align*}

To prove the proposition, let's recall the Convex Hull Value Iteration (CHVI) algorithm~\citep{BarrettNarayanan08}. Note that the CHVI algorithm iteratively updates the value function for each state by considering the convex hull of the achievable rewards via
\begin{align*}
    \vvf(s) = \max_{a \in \Ac} \sum_{s' \in \St} \T(s' \mid s, a)  \text{CH} \left(\R(s, a) + \gamma \vvf(s')\right),
\end{align*}
where $\text{CH}(\cdot)$ denotes the convex hull operation. This update rule ensures that the value function $\vvf(s)$ lies within the convex hull of the achievable rewards and the $\text{CH}(\cdot)$ achievable value functions ${\vvf^{\pi}(s) \mid \pi \in \Pi}$ forms the CCS. Therefore, for any linear preference vector $\w$, there must exist at least a policy $\pi_{\text{CSS}}$ such that
\begin{align*}
    \w^T \vvf^{\pi_{\text{CCS}}}(s) = \max_{\pi \in \Pi} \w^T \vvf^\pi (s),~ \forall s \in \St.
\end{align*}
The resulting policies form the CSS, which are sufficient to cover all linear preferences $\w \in \Omega$.
Thus, for any linear combination of objectives, the optimal policy can be found within the CSS, confirming its sufficiency and optimality.
\end{proof}

While GGF introduces non-linear fairness objectives, its piecewise linearity and concavity allow representation as a maximum of linear functions, which ensures that solutions lie within the CCS. The following proposition establishes the sufficiency of the CCS in representing optimal policies for $\ggf$ preference weights.

\begin{proposition} \label{prop:CCS}
For any $s \in \St$ in an MOMDP and a piecewise-linear concave welfare function $\ggf$ (e.g., GGF) that can be represented as, $\ggf(\vvf^\pi(s)) = \min_{\sigma \in \SD_{\nO}} \bigl\{ \w^\top_{\sigma} \vvf^\pi(s) \bigr\}$,  there exists a policy $\pi^* \in \text{CCS}$ such that:
\begin{align*}
    \ggf(\vvf^{\pi^*}(s)) \geq \ggf(\vvf^\pi(s)) \quad \forall \pi \in \Pi.
\end{align*}
\end{proposition}

\begin{proof}
    Consider an arbitrary permutation $\sigma_A \in \SD_{\nO}$. Since $\ggf$ is a piecewise-linear and concave function, under a fixed permutation $\sigma_A$ it becomes:
    \begin{align*}
        \ggf(\vvf^\pi(s)) &= \w_{\sigma_A}^\top \vvf^\pi(s).  
    \end{align*}
    Let $\pi_A \in \Pi$ be the policy that maximizes this linear scalarization:
    \begin{align*}
        \pi_A &= \argmax_{\pi \in \Pi} \omega_{\sigma_{A}}^\top \vvf^{\pi}(s).
    \end{align*}
    By the definition of CCS and the result from~\cref{le: op in CCS}, there exist a $\pi^* \in \text{CCS}$ such that
    \begin{align*}
        \phi_{\w_{\sigma_{A}}}(\vvf^{\pi^*}(s)) &\geq \phi_{\w_{\sigma_{A}}}(\vvf^{\pi_A}(s)) .
    \end{align*}
    Thus,
    \begin{align*}
        \phi_{\w_{\sigma_{A}}}(\vvf^{\pi^*}(s)) &\geq \phi_{\w_{\sigma_{A}}}(\vvf^{\pi}(s)) \quad \forall \pi \in \Pi
    \end{align*}
    Because this holds for any permutation $\sigma \in \SD_{\nO}$, we cam conclude that for any policy $\pi \in \Pi$, there exists a corresponding $\pi^* \in \text{CCS}$ such that
    \begin{align*}
        \forall \pi \in \Pi,\quad \exists \pi^* \in CCS, \quad
        \ggf(\vvf^{\pi^*}(s)) &\geq \ggf(\vvf^\pi(s)).
    \end{align*}
\end{proof}

\paragraph{\textbf{Fairness of Non-Stationary Policies.}} In fair MORL, learning non-stationary policies can be particularly beneficial, as they leverage historical information to make more informed decisions and adapt over time (see~\Cref{ex: non-stationary}).  

\begin{proposition}\label{pro:nonstationary}
Let the reward $\R$ be nonnegative, and $\Pi_S$ and $\Pi_{NS}$ be the sets of stationary and non-stationary policies, respectively. For any $s \in \St$ in an MOMDP and a given $\ggf$, there exists a non-stationary policy $\pi_{NS} \in \Pi_{NS}$ that achieves a higher welfare score than any stationary policy $\pi_S \in \Pi_S$, i.e.,
\begin{align*}
    \exists \space \pi_{\text{NS}} \in \Pi_{\text{NS}} : \ggf(\vvf^{\pi_{\text{NS}}}(s)) \geq \max_{\pi_{\text{S}} \in \Pi_{\text{S}}} \ggf(\vvf^{\pi_{\text{S}}}(s))
\end{align*}
\end{proposition}

\begin{proof}
    Let the state value function be defined by:
    \begin{align*}
        \vvf(s) =\mathbb{E}\left[\vG_{t}\Big|\, s_t = s\right]
    \end{align*} 
    where the return $\vG_{t}$ is given by:
    \begin{align*}
        \vG_{t} = \sum_{k=0}^{\infty} \gamma^k\, \R_{t+k+1}.
    \end{align*} 
    Suppose an episode begins at time $t$ and terminates at time $T_{\text{end}}$. For any intermediate time $T$ with $t \leq T < T_{\text{end}}$, we can decompose the return into two parts:
    \begin{align*}
        \vG_t = \underbrace{\R_{t+1} + \gamma \R_{t+2} + \dots + \gamma^{T-t-1} \R_T}_{\vG_t^{(1)}} + \underbrace{\gamma^{T-t} \left( \R_{T+1} + \gamma \R_{T+2} + \dots \right)}_{\vG_t^{(2)}}.
    \end{align*}
    With above decomposition, We define value function as two parts:
    \begin{align*}
        \text{Early-period value function: }\vvf_1(s)=\mathbb{E}\left[\vG^{(1)}_{t}\Big|\, s_t = s\right]
    \end{align*}
    \begin{align*}
        \text{Late-period value function: }\vvf_2(s)=\mathbb{E}\left[\vG^{(2)}_{t}\Big|\, s_T = s\right]
    \end{align*}
    so that
    \begin{align*}
        \vvf(s) = \vvf_1(s) + \gamma^{T-t} \vvf_2(s)
    \end{align*}
    At time $T$, stationary policy $\pi_S$ selects action solely based on late period value function $\vvf_2(s)$, while non-stationary policy has access to both early $\vvf_1(s)$ and late period value function $\vvf_2(s)$ and can condition its action selection on the combined information given by two value functions.\\
    Under a stationary policy, The total value can be presented as:
    \begin{align*}
        {\vvf^{\pi_S}(s)} = \vvf_1(s) + \gamma^{T-t} \argmax_{\vvf_2(s)}\{\ggf[\vvf_2(s)]\}
    \end{align*}
    In contrast, under a non-stationary policy the total value is given by
    \begin{align*}
        {\vvf^{\pi_{NS}}(s)} = \argmax_{\vvf_1(s), \vvf_2(s)}\{\ggf[\vvf_1(s) + \gamma^{T-t} \vvf_2(s)]\}
    \end{align*}
    therefore:
    \begin{align*}
        \exists \space \pi_{\text{NS}} \in \Pi_{\text{NS}} : 
        \ggf(\vvf^{\pi_{\text{NS}}}(s)) \geq \max_{\pi_{\text{S}} \in \Pi_{\text{S}}} \ggf(\vvf^{\pi_{\text{S}}}(s))
    \end{align*}
    This completes the proof.
\end{proof}

\paragraph{\textbf{Optimality of Stochastic Policies for Fairness}}
Unlike the single-objective scenario, in MORL, a deterministic policy may not be optimal. A fairer solution can often be achieved through randomization.

\begin{proposition} \label{prop:stochastic}
Let $\Pi_{\text{ST}}$ be the set of stochastic policies and $\Pi_{\text{D}}$ be the set of deterministic policies.
For an MOMDP $\Md$ and a concave welfare function such as $\ggf$, there exists a stochastic policy $\pi_{\text{ST}} \in \Pi_{\text{ST}}$ such that:
\begin{align*}
\ggf(\vvf^{\pi_{\text{ST}}}) \geq \max_{\pi_{\text{D}} \in \Pi_{\text{D}}} \ggf(\vvf^{\pi_{\text{D}}}).
\end{align*}
\end{proposition}

\begin{proof}
The key idea here is that a stochastic policy can represent a convex combination of deterministic policies for any concave welfare function $\ggf$~\cite{busa2017multi}. Hence, stochastic policies can achieve outcomes in the objective space that are unattainable by deterministic policies. Specifically, for $\ggf$, a deterministic policy $\pi_{\text{D}}$ yields a fixed utility vector $\vvf^{\pi_{\text{D}}}$ while a stochastic policy $\pi_{\text{ST}}$ can yield a distribution over utility vectors. Thanks to concavity of $\ggf$, which makes our problem in~\ref{eq:fair-morl} convex optimization and Jensen's inequality~\citep{Jensen67}, we obtain
\begin{align}\label{eq:jensen2}
\ggf\left(\Expect_{\tau \sim \pi} [\vvf^{\pi_{\text{st}}}]\right) \ge \Expect_{\tau \sim \pi} \left[\ggf(\vvf^{\pi_{\text{st}}})\right].
\end{align}
Since $\ggf$ is a piecewise linear concave function, there exists a stochastic policy $\pi_{\text{st}}$ that is a convex combination of deterministic policies such that
\begin{align}\label{eq:expect of swf2}
\Expect_{\tau \sim \pi}[\swf(\vvf^{\pi_{\text{st}}})] \geq \max_{\pi_{\text{d}} \in \Pi_{\text{D}}} \swf(\vvf^{\pi_{\text{d}}}).    
\end{align}
By combining~\eqref{eq:jensen2} and ~\eqref{eq:expect of swf2}, we can obtain
\begin{align*}
    \swf(\Expect_{\tau \sim \pi}[\vvf^{\pi_{\text{st}}}]) \geq \Expect_{\tau \sim \pi}[\swf(\vvf^{\pi_{\text{st}}})] \geq \max_{\pi_{\text{d}} \in \Pi_{\text{D}}} \swf(\swf^{\pi_{\text{d}}}) ].
\end{align*}
This completes the proof.
\end{proof}

The optimality of stochastic policies implies that restricting the search for fair solutions to deterministic policies is insufficient. Stochastic policies offer a broader range of solutions and may better capture the trade-offs among multiple objectives, enhancing the overall fairness of the policy.

\section{Fairness} \label{app:fairness}
In a fair single-policy setting, where the goal is to learn a single policy treating all users equally, three fairness principles, efficiency, equity, and impartiality, are defined below.

\begin{definition}
Efficiency states that among two feasible solutions, if one solution is (weakly or strictly) preferred by all users, then it should be preferred to the other one, e.g., $\bm u \succ \bm u^\prime \Rightarrow \swf(\bm u) > \swf(\bm u^\prime)$, where $\swf(\bm u)$ is the scalar utility function that specifies the value of a solution.
\end{definition}
Intuitively, the efficiency property specifies that given all else equal, one prefers to increase a user’s utility. In the MORL setting, the efficiency property simply means Pareto dominance. More specifically, a solution is considered efficient if it is not dominated by any other solution for all objectives.

Next, we discuss the significance of the \textit{equity} property, which is a stronger property than efficiency and is often associated with distributive justice, as it refers to the fair distribution of resources or opportunities.
This property ensures that a fair solution follows the \textit{Pigou-Dalton principle}~\citep{Moulin04}, which states the transferring of rewards from the more advantaged users to the less advantaged users. 

\begin{definition}
A solution satisfies the \textit{Pigou-Dalton principle} if for all $\bm u$, $\bm u^\prime$ equal except for $u_i=u_i'+\delta$ and  $u_j=u_j'-\delta$ where $u_i'-u_j' > \delta > 0$, $\swf(\bm u) > \swf(\bm u^\prime)$.
\end{definition}

Finally, we discuss the \textit{impartiality} property. This property is rooted in the principle of ``equal treatment of equals'', which states that individuals sharing similar characteristics should be treated similarly.  

\begin{definition}
In a system, individuals with similar characteristics should be treated similarly, i.e., the solution should be independent of the order of its arguments $\swf(\bm u) = \swf(\bm u_{\sigma})$, where $\sigma$ is a permutation and $\bm u_{\sigma}$ is the vector obtained from vector $\bm u$ permuted by $\sigma$. 
\end{definition}

\subsection{Welfare Function}
A welfare function, denoted as $\swf: \mathbb R^D \to \mathbb R$, aggregates the utilities of all users (or objectives) and offers a metric of the overall desirability of a solution for the entire group, where $\w$ represents the set of aggregation weights for all objectives.
One well-established welfare function used in this paper is the generalized Gini welfare function. The generalized Gini welfare function constitutes a specific instance of the ordered weighted average (OWA)\citep{Yager88}. It is a renowned welfare function employed in multi-objective optimization~\citep{Weng19,SiddiqueWengZimmer20,zimmer2021learning,do2022optimizing,YuSiddiqueWeng23,YuSiddiqueWeng23ECAI,siddique2023fairness}, initially devised to quantify income distribution inequality in economics \citep{Weymark81}.
The generalized Gini welfare function is defined as follows:
\begin{align}\label{eq: ggi app}
    \ggf(\bm u) = \sum_{i=1}^{\nO} \omega_{\sigma (i)} u = \bm w^T_{\sigma} \bm u \,,
\end{align}
where $\sigma \in \mathbb S_{\nO}$, which depends on $\w$, is the permutation that sorts the components of $\w$ and $\w_\sigma = (\omega_{\sigma(1)}, \ldots, \omega_{\sigma(\nO)})$. 
Equation~\eqref{eq: ggi app} holds as the weights are rearranged based on the utility vector, assigning the largest weight to the smallest component of $\bm u$, the second-largest weight to the second-smallest component of $\bm u$, and so forth.

The generalized Gini welfare function satisfies the three fairness properties. Due to the positive weights, it is monotonically related to Pareto dominance, fulfilling the efficiency property. Moreover, the reordering of the components in the welfare function makes it symmetric with respect to its components, satisfying the impartiality property. Lastly, as the generalized Gini weights are positive and decreasing, it is Schur-concave, meeting the equity property.

Among numerous welfare functions, the generalized Gini welfare function possesses several favorable properties, namely, simplicity as it is a weighted sum in the Lorenz space~\citep{chakravartyEthicalSocialIndex1990,PernyWengGoldsmithHanna13UAI}, well-understood properties axiomatized by  \citet{Weymark81}, and generality. These favorable properties make it a suitable choice for addressing the challenge of finding fair solutions. Moreover, it is notably a concave function, which will make the solution to our problem easier.

To emphasize the versatility of the generalized Gini welfare function, various special cases can be derived by adjusting its weights accordingly. These cases include:
\begin{itemize}
    \setlength\itemsep{0pt}
    \setlength\parskip{0pt}
    \item \textbf{Maxmin fairness}: Setting $\omega_1=1$ and $\omega_i=0$ for $i=2,\cdots,K$ corresponds to the maxmin notion of fairness~\citep{Rawls71}.
    \item \textbf{Regularized maxmin fairness}: Assigning $\omega_1=1$ and $\omega_i=\varepsilon$ for $i=2,\cdots,K$ aligns with the regularized maxmin notion of fairness.
    \item \textbf{Utilitarian approach}: Setting $\omega_i=1/K$ represents the utilitarian approach.
    \item \textbf{Leximin fairness}: If the ratio $\omega_j/\omega_{j+1}$ tends toward infinity, it corresponds to the leximin notion of fairness~\citep{Rawls71,kurokawaLeximinAllocationsReal2015}.
\end{itemize}

\section{Descriptions of Environments}
\subsection{Species Conservation}
In the field of ecology, the challenge of conserving interdependent endangered species is paramount. The simulation environment focuses on the balance required in the conservation of two such species: the sea otter and the northern abalone, which are currently endangered. The predation relationship between these species, with sea otters feeding on abalones, presents a unique challenge that requires careful consideration of fairness and equity in conservation efforts. Based on the framework in~\citep{chades2012setting}, we define the state space as the current population numbers of the sea otters and northern abalones. The action space consists of: introducing sea otters, enforcing antipoaching measures, controlling sea otter populations, implementing a combination of half-antipoaching and half-controlled sea otters, or taking no action. Each action carries significant ecological consequences; for instance, while the reintroduction of sea otters is essential for maintaining the abalone population, it must be carefully managed to prevent the abalone's extinction. Conversely, overlooking other management actions could lead to the demise of either species. The transition function employed in our model accounts for population dynamics, including external threats such as poaching and oil spills. Since our objective is to optimize the population densities of both species, we define the reward function as the densities of both species, i.e., $\nO=2$.

\subsection{Resource Gathering}
In this scenario of resource gathering, we consider a $5 \times 5$ grid world domain inspired from~\citep{BarrettNarayanan08}. This domain presents a unique challenge centered around the acquisition of three types of resources: gold, gems, and stones, thereby establishing a multi-objective framework with $\mathcal{K}=3$. The autonomous agent is positioned within this grid world, and resources are distributed randomly across various locations. As a resource is collected by the agent, it is immediately regenerated at a new random location within the grid, ensuring a perpetual availability of resources.
In this problem, the state is characterized by the agent's current location on the grid and a cumulative count of each type of resource collected over the course of the agent's trajectory. The agent can navigate the grid through actions aligned with the four cardinal directions: up, down, left, and right, facilitating movement across the grid.
To add complexity to the resource management challenge, resources are assigned differing values, reflecting their relative importance. Specifically, gold and gems are attributed a value of 1, underscoring their significance, whereas stones are considered less valuable, with a value of 0.4. This valuation leads to an intentionally uneven distribution of resources within the grid, comprising two stones, one gold, and one gem. This configuration is designed to simulate a scenario where the agent must not only maximize the collection of resources but also achieve a balanced acquisition across the different types of resources.
The overarching objective for the agent in this environment is dual: to maximize the total value of resources collected while ensuring an equitable collection across the various resource types. Achieving this balance is crucial for optimizing the agent's resource-gathering strategy, enhancing its overall utility and adaptability within the dynamic grid world. This nuanced approach to resource management in a simulated environment offers insights into the complexities of resource distribution and acquisition strategies, contributing to the broader discourse on multi-objective optimization in dynamic settings.
\subsection{Multi-Product Web Advertising}
We now consider the multi-product web advertising (MWP) problem, where an online store offers $\nO$ distinct types of products for sale and an intelligent agent makes strategic decisions at each timestep about which advertisement to display: a product-specific advertisement for one of the products $i \in [0, ..., \nO-1]$, or a general advertisement that is not tailored to any specific product. The effectiveness of an advertisement is contingent upon its relevance to the customer's recent web activity, with appropriate advertisements significantly increasing the likelihood of a purchase, whereas inappropriate ones may deter the customer altogether. 
The state space of this problem is defined by the number of products available in the store, augmented by the number of visits, purchases, and exits. A visit state indicates a customer's interest in a particular product, a purchase state signifies the completion of a transaction, and an exit state occurs when a customer leaves the website without making a purchase. The action space is expanded to $n+1$ actions, where actions $0$ through $n$ correspond to displaying advertisements for specific products, and action $n$ represents the option to show a general advertisement that does not target any specific product in the inventory. This additional action introduces an additional layer of complexity, as the agent must decide the optimal moment to transition between states.
The reward function is designed such that the agent receives a reward of 1 in the $i^{th}$ dimension of the reward vector if a product of type $i$ is sold after the display of its advertisement.
The primary objective of this problem is to maximize the aggregate returns from product sales while striving for an equitable distribution of sales across the different product types. This goal underscores the need for fair solutions that not only optimize overall profitability but also ensure a balanced representation of product sales, thereby addressing the dual challenges of efficiency and equity in this domain.

\section{Hyperparameters}
\label{app: hyperparameters}
To ensure reproducibility, we have meticulously documented all hyperparameters across different environments in Tables 1,2,3, and 4.
We utilize the well-known high-quality MORL baselines\footnote{https://github.com/LucasAlegre/morl-baselines} for implementing baseline algorithms. In these tables, we present the hyperparameters corresponding to Envelope, GPI, PCN, and our proposed algorithms in three distinct environments, namely, species conservation (SC), resource gathering (RC), and multi-web product advertising (MWP).

\begin{table*}[ht]
\centering
\caption{Set of hyperparameters used for training Envelope.} 
\vspace{8pt}
\label{tab:hyperpara-envelope}
\begin{tabular}{lccc} \toprule
Hyperparameter & SC & RC & MWP \\
\midrule
Discount factor ($\gamma$)    & 0.99 & 0.99 & 0.99  \\
Learning rate ($\alpha$) & 0.0001 & 0.0005 & 0.005  \\
Batch size & 64 & 64 & 64 \\
Hidden Layers    & 256 x 256 x 256 x 256 & 256 x 256 x 256 x 256 & 256 x 256 x 256 x 256   \\
Buffer Size & 50000 & 50000 & 50000 \\
Initial Epsilon & 1.0 & 1.0 & 1.0 \\
Final Epsilon & 0.05 & 0.05 & 0.05 \\
Epsilon Decay Steps & 50000 & 50000 & 50000 \\
Learning Starts & 100 & 100 & 100 \\
Gradient Updates & 1 & 1 & 5 \\
Max Gradient Norm & 1.0 & 1.0 & 1.0 \\
$\Omega$ Distribution & Gaussian & Gaussian & Gaussian \\
Tau & 0.5  & 0.5  & 0.5 \\
\bottomrule
\end{tabular}
\end{table*}

\begin{table*}[ht]
\centering
\caption{Set of hyperparameters used for training our proposed methods.} 
\vspace{8pt}
\label{tab:hyperpara-ours}
\begin{tabular}{lccc}
\toprule
Hyperparameter & SC & RC & MWP \\
\midrule
Discount factor ($\gamma$)    & 0.99 & 0.99 & 0.99  \\
Learning rate ($\alpha$) & 0.0001 & 0.0005 & 0.005  \\
Batch size & 64 & 64 & 64 \\
Hidden Layers    & 256 x 256 x 256 x 256 & 256 x 256 x 256 x 256 & 256 x 256 x 256 x 256   \\
Buffer Size & 50000 & 50000 & 50000 \\
Initial Epsilon & 1.0 & 1.0 & 1.0 \\
Final Epsilon & 0.05 & 0.05 & 0.05 \\
Epsilon Decay Steps & 50000 & 50000 & 50000 \\
Learning Starts & 100 & 100 & 100 \\
Gradient Updates & 1 & 1 & 5 \\
Max Gradient Norm & 1.0 & 1.0 & 1.0 \\
$\Omega$ Distribution & Gaussian & Gaussian & Gaussian \\
Tau & 0.5  & 0.5  & 0.5 \\
\bottomrule
\end{tabular}
\end{table*}

\begin{table*}[ht]
\centering
\caption{Set of hyperparameters used for training GPI.} 
\vspace{8pt}
\label{tab:hyperpara-gpi}
\begin{tabular}{lccc}
\toprule
Hyperparameter & SC & RC & MWP \\
\midrule
Discount factor ($\gamma$)    & 0.99 & 0.99 & 0.99  \\
Learning rate ($\alpha$) & 0.0001 & 0.0005 & 0.005  \\
Batch size & 128 & 128 & 256 \\
Hidden Layers    & 256 x 256 x 256 x 256 & 256 x 256 x 256 x 256 & 256 x 256 x 256 x 256   \\
Num Networks & 2 & 2 & 2 \\
Buffer Size & 50000 & 50000 & 50000 \\
Initial Epsilon & 1.0 & 1.0 & 1.0 \\
Final Epsilon & 0.05 & 0.05 & 0.05 \\
Epsilon Decay Steps & 50000 & 50000 & 50000 \\
Learning Starts & 100 & 100 & 100 \\
Gradient Updates & 1 & 1 & 5 \\
\bottomrule
\end{tabular}
\end{table*}

\begin{table*}[ht]
\centering
\caption{Set of hyperparameters used for training PCN.} 
\vspace{8pt}
\label{tab:hyperpara-pcn}
\begin{tabular}{lccc}
\toprule
Hyperparameter & SC & RC & MWP \\
\midrule
Discount factor ($\gamma$)    & 0.99 & 0.99 & 0.99  \\
Learning rate ($\alpha$) & 0.0001 & 0.0001 & 0.0005  \\
Batch size & 128 & 256 & 128 \\
Hidden Layers    & 64 x 64 & 64 x 64 & 64 x 64  \\
Desired Return & [1, 1] & [200, 200, 200]  & [100, 100, 100, 100, 100]      \\ 
Buffer Size & 500000 & 500000 & 1000000 \\
Max Horizon & 5000 & 1000 & 1000 \\
\bottomrule
\end{tabular}
\end{table*}

\end{document}